\documentclass{article}

 \usepackage[preprint, nonatbib]{nips_2018}

\usepackage[utf8]{inputenc} 
\usepackage[T1]{fontenc}    
\usepackage{hyperref}       
\usepackage{url}            
\usepackage{booktabs}       
\usepackage{amsfonts}       
\usepackage{nicefrac}       
\usepackage{microtype}      
\usepackage{amsmath, amssymb, amsthm, mathtools, mathrsfs}
\usepackage{dsfont}
\usepackage{xcolor}
\usepackage{subfig}
\usepackage{graphicx}
\usepackage{booktabs}
\usepackage{graphicx}
\usepackage[algo2e,ruled]{algorithm2e}
\usepackage{enumitem}
\theoremstyle{plain}
\newtheorem{theorem}{Theorem}[section]

\newtheorem{corollary}[theorem]{Corollary}
\newtheorem{lemma}[theorem]{Lemma}

\theoremstyle{definition}

\theoremstyle{remark}

\newtheorem{remark}{Remark}[section]

\def\E{\mathbb{E}}

\newcommand*{\R}{\mathbb{R}}      
 
\newcommand*{\huaF}{\mathcal{F}}  

\newcommand*{\sfN}{\mathsf{N}}    
\newcommand*{\bigO}{\mathcal{O}}  
\newcommand*{\huaR}{\mathcal{R}}  
\newcommand*{\Hil}{\mathcal{H}}   
\newcommand*{\eps}{\varepsilon}   
\newcommand*{\transpose}{\mathrm{T}}

\DeclareMathOperator*{\argmax}{\arg\max}
\DeclarePairedDelimiter{\bk}{(}{)} 
\DeclarePairedDelimiter{\abs}{\lvert}{\rvert}
\DeclarePairedDelimiter{\norm}{\lVert}{\rVert}
\DeclarePairedDelimiter{\hua}{\{}{\}}             

\newcommand*{\vect}[1]{\mathbf{#1}}
\newcommand\numberthis{\addtocounter{equation}{1}\tag{\theequation}}
\makeatletter
\let\oldbk\bk
\def\bk{\@ifstar{\oldbk}{\oldbk*}}
\makeatother

\hypersetup{
    colorlinks,
    linkcolor={red!50!black},
    citecolor={blue!50!black},
    urlcolor={blue!80!black}
}
\allowdisplaybreaks

\title{On the Study of Sample Complexity for Polynomial Neural Networks}

\author{
Chao Pan\\
Department ECE\\
UIUC\\
\texttt{chaopan2@illinois.edu}\\
\And
Chuanyi Zhang\\
Department ECE\\
UIUC\\
\texttt{chuanyi5@illinois.edu}\\
}

\begin{document}

\maketitle

\begin{abstract}
  As a general type of machine learning approach, artificial neural networks have established state-of-art benchmarks in many pattern recognition and data analysis tasks. Among various kinds of neural networks architectures, polynomial neural networks (PNNs) have been recently shown to be analyzable by spectrum analysis via neural tangent kernel, and particularly effective at image generation and face recognition. However, acquiring theoretical insight into the computation and sample complexity of PNNs remains an open problem. In this paper, we extend the analysis in previous literature to PNNs and obtain novel results on sample complexity of PNNs, which provides some insights in explaining the generalization ability of PNNs.
\end{abstract}
\section{Introduction}
Neural networks are powerful tools in the area of machine learning, especially deep learning. And it is observed that we can normally get better performance when we increase the number of parameters in neural networks, such as the width (the dimensions of parameter matrices each layer) and the depth (the number of layers). By intuition, when the number of parameters grows neural networks will have the potential to overfit the training data, resulting in a larger generalization error. However, it is surprising that neural networks as deep as $1,000$ layers can still possess remarkable generalization ability~\cite{li2021training}. This inspires researchers to study the representation power of neural networks.

It is well-known that generalization error is related to sample complexity of the function class of neural networks~\cite{bousquet2003introduction}. Therefore, most works start from the analysis of sample complexity in neural networks. In~\cite{neyshabur2015norm}, by inductively applying contraction principle on single layer of a neural network, it gives a bound on the network generalization error which has exponential dependence on the network depth $d$. With further investigation,~\cite{bartlett2017spectrally} shows that this bound can be tightened to polynomial dependence on $d$, and~\cite{zhu2022controlling} studies the sample complexities of the Coupled CP-decomposition model (CCP) and Nested Coupled CP-decomposition model (NCP) of polynomial neural networks (PNNs), which also have polynomial dependence on the number of layers. However, in practice, people have observed that the number of training samples required for convergence does not necessarily goes up when a network becomes deeper.~\cite{golowich2018size} proves that, under certain conditions on norms of weight matrices and assumptions on the activation functions, we can achieve a depth-independent bound on the sample complexity.

In this paper, we mainly focus on developing a similar depth-independent bound on sample complexity of a more general class of PNNs, which is quite different from the architectures that previous works have considered. Specifically, we use $\sigma(x)=x^k\ (k\geq 2)$ as our activation function. There are two main reasons why PNNs are of particular interest in this paper. First, PNNs are easy to analyze because of the simple structure.~\cite{chrysos2020p, choraria2022spectral} study the representation power of one instance of PNNs called $\Pi$-Nets.~\cite{livni2014approx,soltani2018towards,du2018power} show that PNNs can efficiently approximate other networks. Specifically, it only requires $\bigO(\log\log(1/\epsilon))$ polynomial activation layers to approximate a 2-layer sigmoid networks with $l_1$ regularization, where $\epsilon$ indicates the approximation error. Beside this,~\cite{livni2014approx} shows that the training technique for such networks is very similar to SGD, and can be computed efficiently for each iteration. The other reason is that we are still lacking an understanding of the generalization ability of PNNs. Therefore, our contribution in this paper is an in-depth study on the sample complexity of PNNs, and we show that the complexity can be depth-independent under certain conditions, which may shed some light in understanding the recent success of deep PNNs in image generation~\cite{karras2019style} and face recognition~\cite{chrysos2020p}. Note that most previous works on PNNs only focus on one specific choice of $k$ (i.e., $k=2$), but our conclusion in this paper is applicable to any $k\geq 2$.

\section{Preliminaries}
This section introduces the notations used in later context. We use bold-faced lower case letters for vectors, non-bold lower case letters for scalars, and capital letters for matrices or fixed parameters.

For a vector $\mathbf{w}\in \mathbb{R}^n$, its $L_p$ norm is denoted as $\|\vect{w}\|_p=\bk{\sum_{i=1}^n|w_i|^p}^{\frac{1}{p}}$, and $\|\vect{w}\|=\|\vect{w}\|_2$. For a matrix $W \in\mathbb{R}^{h\times n}$, 
$\norm*{W}_{q,p} = \bk{\sum_{j}\bk{\sum_k |W_{j,k}|^q}^{\frac{p}{q}}}^{\frac{1}{p}} = \bk{\sum_{j}\norm*{\vect{w}_j}_q^{p}}^{\frac{1}{p}}$, where $\vect{w}_j$ is the $j$-th row of matrix $W$. Note that our definition of $\norm*{W}_{q,p}$ is a bit different from the $L_{p,q}$ norm of a matrix, and it is for simplicity of analysis.
$\norm*{W}_F$ refers to Frobenius norm and $\|W\|$ refers to spectral norm. $\norm*{W}_p$ denotes Schatten $p$-norm of the spectrum of $W$ with $p\geq 1$, which is defined by $\norm*{W}_p=\left(\sum_{i}\sigma_i^p(W)\right)^{\frac{1}{p}}$ and $\sigma_i(W)$'s are singular values of $W$.

A cascade neural network with depth $d$ is defined as follows:
\[
	\sfN^{(d)}: \vect{x} \rightarrow W_d\sigma_{d-1}(W_{d-1}\sigma_{d-2}(\dots \sigma_1(W_1\vect{x}))).
\]
And to simplify the notation, we denote $W_b^r$ as the shorthand for the matrix tuple $\{W_b,W_{b+1},\dots,W_r\}$, and $\sfN_{W_b^r}$ denotes the function computed by the sub-network composed of layers $b$ through $r$:
\[
	\sfN_{W_b^r}: \vect{x} \rightarrow W_r\sigma_{r-1}(W_{r-1}\sigma_{r-2}(\dots \sigma_b(W_b\vect{x}))).
\]

Given a real-valued function class $\mathcal{H}$ and some set of data points $\vect{x}_1,\dots,\vect{x}_m\in\mathcal{X}$, the (empirical) Rademacher complexity $\hat{\mathcal{R}}_m(\mathcal{H})$ is
\[
	\hat{\mathcal{R}}_m(\mathcal{H}) = \E_{\eps^m}\sup_{h\in\mathcal{H}}\frac{1}{m}\sum_{i=1}^m\eps_ih(\vect{x}_i),
\]
where $\eps_i$ is the Rademacher random variable, which takes value $+1$ and $-1$ with equal probability. Throughout this paper, $\E$ is a shorthand for $\E_{\eps^m}$.

Our main results provide upper bounds on the Rademacher complexity with respect to PNNs under mild assumptions on input data.

\section{Related Works}
\label{sec:literStudy}
Classical results on the sample complexity of neural networks normally have strong dependency on the dimensions of parameter matrices. The upper bound shown in~\cite{anthony1999neural} is based on VC dimension and strongly rely on both the depth and the width of the network. In~\cite{neyshabur2015norm}, the authors make use of contraction principle and show an upper bound that have an exponential dependence on the network depth $d$, even if the norm of each parameter matrix is strictly upper bounded. Specifically, if input data $\vect{x}_i, i\in[m]$ satisfies $||\vect{x}_i||\leq B$ and each parameter matrix $W_j,j\in[d]$ satisfies $||W_j||_F\leq M_F(j)$, then under some assumptions on the activations, the results in~\cite{neyshabur2015norm} show that
$$
\hat{\huaR}_m\bk{\huaF\bk{X^m}} \sim \bigO\bk{\frac{ B2^d\prod_{j=1}^d M_F(j) }{ \sqrt{m}}},
$$
where $X^m$ is a matrix consisting of all input data points. Although this bound has no explicit dependence on the network width (the dimensions of $W_j,j\in[d]$), it has an exponential dependence on the network depth $d$, even if $M_F(j)\leq 1, \forall j\in[d]$, which is obviously not desirable.~\cite{bartlett2017spectrally} improves this exponential dependency to polynomial dependency based on a covering numbers argument. The authors show that 
$$\hat{\huaR}_m\bk{\huaF\bk{X^m}} \sim \tilde{\bigO} \bk{ B\prod_{j=1}^d \norm*{W_j} \sqrt{\frac{d^{3} }{ m}} },
$$ 
where some lower-order logarithmic factors are ignored. Although the dependency of $d$ becomes polynomial (i.e., $d^{3/2}$) in the new bound, the bound becomes trivial when $d\geq \Omega(m^{1/3})$, since the bound will not decrease with the growth of sample size $m$ in this case, indicating that it is not possible to reduce generalization error by increasing the size of training data. This is obviously not consistent with our observations during the training process of deep neural networks~\cite{li2021training}.


In~\cite{golowich2018size} the authors show another line of proof techniques that gives rise to the first \emph{depth-independent} upper bound on sample complexity of neural networks. Their main idea follows these two observations:
\begin{itemize}
\item Each single layer of a neural network is equivalent to a generalized linear classifier of the form $\left\{\mathbf{x} \mapsto \sigma\left(\mathbf{w}^{\top} \mathbf{x}\right):\|\mathbf{w}\| \leq M,\|\vect{x}\|\leq B\right\}$, where $\sigma(\cdot)$ is some common activation function such as ReLU. And the generalization error of this function class is known to be $\bigO\bk{\frac{MB}{\sqrt{m}}}$.
\item We can also view the classifier stated above as a class of ``ultra-thin'' network with following form:
\[ \huaF:=\left\{\sigma_d(w_d\sigma_{d-1}(w_{d-1}\sigma_{d-2}(\dots \sigma_1(\vect{w}_1^{\top}\vect{x})))), \|\vect{w}_1\|\cdot\prod_{j=2}^d |w_j|\leq M\right\}. \]
Note that here $\vect{w}_1$ is a vector and $w_2,\dots,w_d$ are scalars. Therefore, for this ultra-thin network, we should still have $\hat{\huaR}_m\bk{\huaF\bk{X^m}} \sim \bigO\bk{\frac{MB}{\sqrt{m}}}$. The only difference between real neural networks and this ``ultra-thin'' network is that $W_2,\dots,W_d$ are matrices in real neural networks. Therefore, one could image that the change would take place in the constraint $\prod_{j=1}^d \norm*{W_j}\leq M$, but the upper bound should remain independent of depth $d$.
\end{itemize}

With these two observations,~\cite{golowich2018size} proposes the following two theorems:
\begin{theorem}[\cite{golowich2018size}]\label{thm:alternative-net_3.1}
  For any $p\geq 1$ and any net $\sfN^{(d)}$ such that $\prod_{j=1}^d\norm{W_j}\geq \Gamma $ and $\prod_{j=1}^d\norm{W_j}_p \leq M $, for $\forall r \in [d]$, there exists an alternative net $\widetilde{\sfN}^{(d)} $ such that they are identical except for weights of the $r$-th layer: $\widetilde{W}_r$. 
  $\widetilde{W}_r = s\vect{u}\vect{v}^T $, where $s, \vect{u}, \vect{v}$ are the most significant singular value and corresponding left and right singular vectors of $W_r$.
  \[ \sup_{\vect{x}\in\mathcal{X}} \norm*{\sfN^{(d)}(\vect{x}) - \widetilde{\sfN}^{(d)}(\vect{x}) } \leq B \prod_{j=1}^d \norm{W_j} \bk{\frac{2p\log(M/\Gamma)}{r}}^{1/p}. \]
\end{theorem}
\begin{theorem}[\cite{golowich2018size}]\label{thm:composition-net_3.2}
Let $\Hil: \R^n \to [-R, R] $, $\huaF_{L,a}: [-R, R] \to \R $ are $L$-Lipschitz functions and $f(0) = a$ for some fixed $a$. Denote $\huaF_{L,a} \circ \Hil :=\{f(h(\cdot)):f\in\huaF_{L,a},h\in\Hil\}$. The Rademacher complexity satisfies
  \begin{align*}
  	\hat{\huaR}_m\bk{ \huaF_{L,a} \circ \Hil } & \leq cL\bk{ \frac{R}{\sqrt{m}} + \log^{3/2}(m)\hat{\huaR}_m(\Hil) },
  \end{align*}
  where $c>0$ is some universal constant.
\end{theorem}

With Theorem~\ref{thm:alternative-net_3.1}, we can substitute the original network to a $r$-layer net followed by a univariate net. And Theorem~\ref{thm:composition-net_3.2} shows that we achieve a bound on the sample complexity of neural networks by complexity of the first $r$ layers net. Note that after Theorem~\ref{thm:composition-net_3.2} the complexity is independent of the depth $d$ of the original network. Combining these two we arrive at the depth-independent bound.
\begin{theorem}[\cite{golowich2018size}]\label{thm:golowich_depth_ind_bound}
Consider the following hypothesis class of networks on $\mathcal{X}=\{\vect{x}:\norm{\vect{x}}\leq B\}$:
$$
\mathcal{H}=\left\{\sfN_{W_1^d}:\begin{array}{c}
\prod_{j=1}^{d}\left\|W_{j}\right\| \geq \Gamma \\
\forall j \in\{1, \ldots, d\}, W_{j} \in \mathcal{W}_{j}, \max \left\{\frac{\left\|W_{j}\right\|}{M(j)}, \frac{\left\|W_{j}\right\|_{p}}{M_{p}(j)}\right\} \leq 1
\end{array}\right\},
$$
for some parameters $p,\Gamma\geq 1,\{M(j),M_p(j),\mathcal{W}_{j}\}_{j=1}^d$. Also for any $r\in\{1,\dots,d\}$, define
$$
\mathcal{H}_r=\left\{\sfN_{W_1^r}:\begin{array}{c}
\sfN_{W_1^r} \text{ maps to } \mathbb{R} \\
\forall j\in\{1,\ldots,r-1\}, W_j\in\mathcal{W}_j \\
\forall j \in\{1,\ldots,r\}, \max \left\{\frac{\left\|W_{j}\right\|}{M(j)}, \frac{\left\|W_{j}\right\|_{p}}{M_{p}(j)}\right\} \leq 1
\end{array}\right\}.
$$
Finally, for $m>1$, let $\ell\circ\mathcal{H}:=\{(\ell_1(h(\vect{x}_1)),\ldots,\ell_m(h(\vect{x}_m))):h\in\mathcal{H}\}$, where $\ell_1,\ldots,\ell_m$ are real-valued loss functions which are $\frac{1}{\gamma}$-Lipschitz and satisfy $\ell_1(\mathbf{0})=\cdots=\ell_m(\mathbf{0})=a$ for some $a\in\mathbb{R}$. Assume that $|a|\leq\frac{B\Pi_{j=1}^dM(j)}{\gamma}$. Then the Rademacher complexity $\hat{\huaR}_m\bk{\ell\circ \Hil}$ is upper bounded by
\begin{align*}
\frac{c B \prod_{j=1}^{d} M(j)}{\gamma}\left(\min _{r \in\{1, \ldots, d\}}\left\{\frac{\log ^{3 / 2}(m)}{B} \cdot \max _{r^{\prime} \in\{1, \ldots, r\}} \frac{\hat{\mathcal{R}}_{m}\left(\mathcal{H}_{r^{\prime}}\right)}{\prod_{j=1}^{r^{\prime}} M(j)}+\left(\frac{\log \left(\frac{1}{\Gamma} \prod_{j=1}^{d} M_{p}(j)\right)}{r}\right)^{1 / p}+\frac{1+\sqrt{\log r}}{\sqrt{m}}\right\}\right),
\end{align*}
where $c>0$ is a universal constant.
\end{theorem}

It is worth pointing out that the results show in~\cite{golowich2018size} are not directly applicable to PNNs, due to their assumptions on activations. Therefore, we extend this line of analysis to PNNs in this paper.

\section{Improved Sample Complexity for Polynomial Neural Networks}
Both depth-dependent and depth-independent results are provided in this section. Some proofs are relegated to a longer version of this paper for brevity.
\subsection{Depth-Dependent Sample Complexity for Polynomial Neural Networks}
\label{sec:dependent}
We will need the following Lemma~\ref{modified:lemma1} and Lemma~\ref{modified:lemma2} to get the bound for polynomial networks with dependence on only $\sqrt{d}$.
\begin{lemma}
\label{modified:lemma1}
Let activations be $\sigma(x)=x^k(k \geq 2)$ which is applied element-wise. Then for $\vect{x}\in\mathcal{X}\subset\mathbb{R}^n$, function class $\mathcal{F}:\forall f\in\mathcal{F},\|f(\vect{x})\|_{p}\leq B, \|W\|_{q,p}\leq R$, where $\frac{1}{q}+\frac{1}{p}=1$, and any convex and monotonically increasing function $g:\mathbb{R}\rightarrow[0,\infty)$, we can have that with condition $BR \leq \bk{\frac{1}{k}}^{\frac{1}{k-1}}$,
\begin{align}
  & \E\sup_{f\in\mathcal{F};\|W\|_{q,p}\leq R}g\bk{\norm*{\sum_{i=1}^m\eps_i\sigma(Wf(\mathbf{x}_i))}_p} 
    \leq 2 \E\sup_{f\in\mathcal{F}}g\bk{R\cdot\norm*{\sum_{i=1}^m\eps_if(\vect{x}_i)}_{p}}.
\end{align}
\end{lemma}
\begin{proof}
Denote $\vect{w}_1,\vect{w}_2,\dots,\vect{w}_h$ as the rows of the matrix $W\in\mathbb{R}^{h\times n}$. Since the activation is $\sigma(x)=x^k$, we have for any $\vect{w}_j$, 
\[
	\eps_i\sigma(\vect{w}_j^\mathrm{T}f(\vect{x}_i)) = \eps_i\norm*{\vect{w}_j}_q^k\sigma\bk{\frac{\vect{w}_j^\mathrm{T}}{\norm*{\mathbf{w}_j}_q}f(\vect{x}_i)}.
\]
Then for the whole matrix $W$ we have 
\begin{equation*}
	\norm*{\sum_{i=1}^m\eps_i\sigma(Wf(\vect{x}_i))}_p^p=\sum_{j=1}^h\norm*{\vect{w}_j}_q^{kp}\bk{\sum_{i=1}^m\eps_i\sigma\bk{\frac{\vect{w}_j^\mathrm{T}}{\norm*{\mathbf{w}_j}_q}f(\vect{x}_i)}}^p
\end{equation*}
According to definition of $\norm*{W}_{q,kp}$, we have 
\[
	\norm*{W}_{q,kp} = \bk{\sum_{j}\bk{\sum_k |W_{j,k}|^q}^{\frac{kp}{q}}}^{\frac{1}{kp}} = \bk{\sum_{j}\norm*{\vect{w}_j}_q^{kp}}^{\frac{1}{kp}} \leq \bk{\sum_{j}\norm*{\vect{w}_j}_q^{p}}^{\frac{1}{p}} = \norm*{W}_{q,p} \leq R,
\]
here we use the fact that $\|\vect{w}\|_p \leq \|\vect{w}\|_q$ if $p\geq q$. So the supremum over $\vect{w}_1,\vect{w}_2,\dots,\vect{w}_h$ such that $\norm*{W}_{q,kp}^{kp}=\sum_j\norm*{\vect{w}_j}_q^{kp}\leq \norm*{W}_{q,p}^{kp} \leq R^{kp}$ must be attained when $\norm*{\vect{w}_j}_q=R$ for some $j$, and $\norm*{\vect{w}_i}_q=0$ for all $i\neq j$, as we can consider $\frac{\vect{w}_j^\mathrm{T}}{\norm*{\vect{w}_j}_q}$ as a unit vector in the sense of $L_q$ norm. Therefore, we have 
\[
	\E\sup_{f\in\mathcal{F},\norm*{W}_{q,p}\leq R}g\bk{\norm*{\sum_{i=1}^m \eps_i\sigma\bk{Wf(\vect{x}_i)}}_p} = \E\sup_{f\in\mathcal{F},\|\vect{w}\|_q=R}g\bk{\abs*{\sum_{i=1}^m \eps_i\sigma\bk{\vect{w}^\mathrm{T}f(\vect{x}_i)}}}.
\]
Since $g(|z|)\leq g(z)+g(-z)$, the above quantity can be upper bounded by
\[
2\E\sup_{f\in\mathcal{F},\|\vect{w}\|_q=R}g\bk{\sum_{i=1}^m \eps_i\sigma\bk{\vect{w}^\mathrm{T}f(\vect{x}_i)}},
\]
where we use the symmetry property of Rademacher variables $\eps_i$. By Eq. (4.20) in~\cite{ledoux1991probability}, we can further bound this term by
\begin{align*}
2\E\sup_{f\in\mathcal{F},\|\vect{w}\|_q=R}g\bk{\sum_{i=1}^m \eps_i\sigma\bk{\vect{w}^\mathrm{T}f(\vect{x}_i)}} &\leq 2\E\sup_{f\in\mathcal{F},\|\vect{w}\|_q=R}g\bk{L\norm*{\vect{w}}_q\norm*{\sum_{i=1}^m \eps_if(\vect{x}_i)}_p} \\
&=
2\E\sup_{f\in\mathcal{F}}g\bk{LR\norm*{\sum_{i=1}^m \eps_if(\vect{x}_i)}_p},
\end{align*}
where $L$ represents the Lipschitz constant for $\sigma\bk{\cdot}$. Since $\sigma(x)$ is continuous, it is equivalent to make sure that $|\sigma'\bk{\vect{w}^\mathrm{T}f(\vect{x}_i)}| \leq L$. Let us consider $L=1$ in the rest of this paper for simplicity,
\begin{align*}
	kx^{(k-1)} &\leq 1 \\
    x &\leq \bk{\frac{1}{k}}^{\frac{1}{k-1}}.
\end{align*}
So if $\|f(\mathbf{x})\|_p \leq B$ and $\|\vect{w}\|_q \leq R$, we have 
\begin{align}
	\abs*{\vect{w}^\mathrm{T}f(\vect{x}_i)} \leq \|\vect{w}\|_q\cdot \|f(\mathbf{x})\|_p \leq BR\leq \bk{\frac{1}{k}}^{\frac{1}{k-1}}.\label{eq:lemma1_constraint}
\end{align}
\end{proof}
\begin{remark}
Lemma~\ref{modified:lemma1} shows that even if we have a very wide layer (i.e., dimension $h$ is large for $W\in\mathbb{R}^{h\times n}$), the sample complexity analysis is the same as the case where $h=1$. This explains why the complexity can be independent of the width of neural networks. And the constraint in Eq.~\eqref{eq:lemma1_constraint} will asymptotically converge to $1$ when $k$ increases, as shown in Figure~\ref{fig:lemma1_ub}. In supplement we consider one special case where $q=1, p=\infty$.
\end{remark}

\begin{figure}[htb]
    \centering
    \includegraphics[width = 0.7\linewidth]{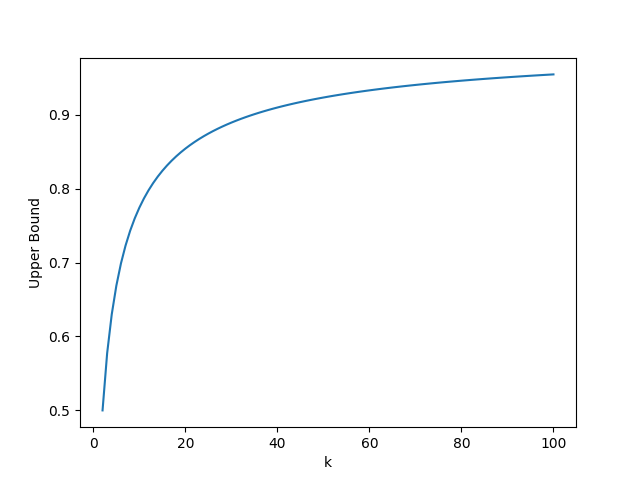}
    \caption{Constraint shown in Eq.~\eqref{eq:lemma1_constraint}.}
    \label{fig:lemma1_ub}
\end{figure}

\begin{lemma}
\label{modified:lemma2}
Let $\mathcal{H}_d$ be the class of real-valued networks of depth $d$ over the domain $\mathcal{X}$, and $\norm{W_i}_{q,p} \leq M(i),\, \forall i \in [d] $, where $\frac{1}{q} + \frac{1}{p} = 1 $, and with activations satisfying Lemma~\ref{modified:lemma1} for each layer. Then,
\begin{align*}
 \hat{\huaR}_m \bk{\Hil_d} \leq \frac{ \prod_{i=1}^d M(i)}{m} \bk{ \sqrt{ 2d\log2}\sqrt{\sum_{i=1}^m \norm{\vect{x}_i}_{p}^2} +  \bk{ \sum_{i=1}^m \norm*{\vect{x}_i}_p^p}^{\frac{1}{p}}}.
\end{align*}
\end{lemma}
\begin{proof}
  For some $\lambda>0$, we can upper bound the empirical Rademacher complexity by
  \begin{align*}
    m \hat{\huaR}_m \bk{\Hil_d} & = \E \sup_{\sfN^{(d-1)}, W_d} \sum_{i=1}^m \eps_i W_d \sigma_{d-1} \bk{ \sfN^{(d-1)}(\vect{x}_i) } \\
      & \stackrel{(a)}{\leq} \frac{1}{\lambda} \log \E \sup_{\sfN^{(d-1)}, W_d} \exp\bk{ \lambda \sum_{i=1}^m \eps_i W_d \sigma_{d-1} \bk{ \sfN^{(d-1)}(\vect{x}_i) }} \\
      & \stackrel{(b)}{\leq} \frac{1}{\lambda} \log \E\sup_{\sfN^{(d-1)}} \exp\bk{ \lambda M(d) \norm*{\sum_{i=1}^m \eps_i \sigma_{d-1} \bk{ \sfN^{(d-1)}(\vect{x}_i)} }_{p} } \\
      & \stackrel{(c)}{\leq} \frac{1}{\lambda} \log \bk{2^d\cdot \E \exp\bk{ \lambda \prod_{i=1}^d M(i) \norm*{\sum_{i=1}^m \eps_i \vect{x}_i }_{p}}} \\
      & = \frac{1}{\lambda}\log\bk{2^d\cdot\E\exp{\lambda Z}} \\
      & = \frac{d\log 2}{\lambda} + \frac{1}{\lambda}\log \bk{\E \exp\bk{ \lambda \bk{Z - \E Z} } } + \E Z,  \numberthis \label{eq:arbitray-dual-norm}
  \end{align*}
  where $M=\prod_{i=1}^d M(i), Z = M\norm*{\sum_{i=1}^m \eps_i \vect{x}_i }_{p}$. Here, $(a)$ follows from Jensen's inequality, $(b)$ follows from Lemma~\ref{modified:lemma1} with $g(x)=\exp\bk{M(d)\cdot \lambda x}$, and $(c)$ follows from repeating this process $d$ times. 
  
  For the second term $\frac{1}{\lambda}\log \bk{\E \exp\bk{ \lambda \bk{Z - \E Z} } }$, we have that 
  $$
    |Z\left(\eps_{1}, \ldots, \eps_{i}, \ldots, \eps_{m}\right)-Z\left(\eps_{1}, \ldots,-\eps_{i}, \ldots, \eps_{m}\right)| \leq 2 M\left\|\mathbf{x}_{i}\right\|_p.
  $$
  This implies that $Z$ is a sub-Gaussian random variable and we can have the concentration inequality
  \[
    \frac{1}{\lambda}\log \E \exp\bk{ \lambda \bk{Z - \E Z} } \leq 
    \frac{1}{\lambda} \frac{\lambda^2 \frac{1}{4} \sum_{i=1}^m \bk{2M\norm{\vect{x}_i}_{p}}^2 }{2} = \frac{\lambda M^2 \sum_{i=1}^m \norm{\vect{x}_i}_{p}^2}{2}.
  \]
  Also by Jensen's inequality, we have
  \begin{align*}
    \E Z &= \E\bk{M\norm*{\sum_{i=1}^m\eps_i\vect{x}_i}_p}=M\cdot\E\bk{\norm*{\sum_{i=1}^m\eps_i\vect{x}_i}_p^p}^{\frac{1}{p}}\\
    & \stackrel{(a)}{\leq} M\cdot \bk{\E \norm*{\sum_{i=1}^m\eps_i\vect{x}_i}_p^p}^{\frac{1}{p}} = M\cdot\bk{\E \bk{\sum_{j=1}^n \abs*{\sum_{i=1}^m x_{ij}\eps_i}^p}}^{\frac{1}{p}} \\
    & = M\cdot\bk{\sum_{j=1}^n \E \bk{\abs*{\sum_{i=1}^m x_{ij}\eps_i}^p}}^{\frac{1}{p}}
    \stackrel{(b)}{\leq} M\cdot \bk{\sum_{i=1}^m \norm*{\vect{x}_i}_p^p}^{\frac{1}{p}}.
  \end{align*}
  Here $(a)$ follows from the fact that $f(x)=x^{1/p}$ is a concave function for $p\geq 1$. Note that $(b)$ is a pretty loose general bound for all $p\geq 1$. A much tighter bound can be obtained if we know the specific choice of $p$ (i.e., $p=2$). 
  
  Combining the results we have that Eq.~\eqref{eq:arbitray-dual-norm} can be further upper bounded as follows
  \begin{align*}
    m \hat{\huaR}_m \bk{\Hil_d} & \leq \frac{d\log 2}{\lambda} + \frac{\lambda M^2 \sum_{i=1}^m \norm{\vect{x}_i}_{p}^2}{2}  + M\cdot \bk{\sum_{i=1}^m \norm*{\vect{x}_i}_p^p}^{\frac{1}{p}}  \\
    & \stackrel{(a)}{\leq} \sqrt{2d\log(2)} M\cdot \sqrt{\sum_{i=1}^m \norm{\vect{x}_i}_{p}^2} + M\cdot \bk{\sum_{i=1}^m \norm*{\vect{x}_i}_p^p}^{\frac{1}{p}},
  \end{align*}
  where we choose $\lambda=\frac{\sqrt{2 d \log (2) }}{M \sqrt{\sum_{i=1}^{m}\left\|\mathbf{x}_{i}\right\|_p^{2}}}$ to optimize the upper bound in $(a)$.
\end{proof}

Combining Lemma~\ref{modified:lemma1} and Lemma~\ref{modified:lemma2}, we can have a depth-dependent upper bound on the sample complexity. But we need to take care of the norm constraints of all $W_i$ to make sure that the condition in Lemma~\ref{modified:lemma1} holds for each layer.
\begin{theorem}
\label{thm:dual_dep}
Let $\mathcal{H}_d$ be the class of real-valued networks of depth $d$ over the domain $\mathcal{X}$. For $j \in \hua{2,\cdots,d}$, $W_j$ has $\norm*{W_j}_{q,p} \leq M(j)$, $M(j) \leq 2^{k-1}$; $\norm{W_1}_{q,p} \leq \frac{1}{2B}$, $\norm{\vect{x}_i}_{p} \leq B$ ; and activation function $\sigma(x)=x^k(k \geq 2)$ which is applied element-wise. Then,
  \begin{align}
    \hat{\huaR}_m \bk{\Hil_d} \leq \frac{ \prod_{i=1}^d M(i)}{m} \bk{ \sqrt{ 2d\log2}\sqrt{\sum_{i=1}^m \norm{\vect{x}_i}_{p}^2} +  \bk{ \sum_{i=1}^m \norm*{\vect{x}_i}_p^p}^{\frac{1}{p}}}.
  \end{align}
\end{theorem}
\begin{proof}
First for function $f(x)=\bk{\frac{1}{k}}^{\frac{1}{k-1}}$, it is monotonically increasing on $[2,\infty)$. So if $k\geq 2$, the condition in Lemma~\ref{modified:lemma1} becomes $BR\leq \frac{1}{2}$, meaning that $R\leq \frac{1}{2B}$. Let us consider the first layer. Since $\|\mathbf{x}_i\|_p\leq B$, we have for sure that 
\[
	\norm{W_1}_{q,p} \leq \frac{1}{2B}.
\]
Then for second layer, the input becomes $f(\mathbf{x}_i)=\sigma_1(W_1\mathbf{x}_i)$. By definition we have 
\begin{align*}
	\norm*{f(\mathbf{x}_i)}_p^p &= \norm*{\sigma_1(W_1\mathbf{x}_i)}_p^p
    = \sum_{j=1}^h\bk{\sigma(\mathbf{w}_j^\mathrm{T}\mathbf{x}_i)}^p
    = \sum_{j=1}^h\bk{\mathbf{w}_j^\mathrm{T}\mathbf{x}_i}^{kp} \\
    &\stackrel{(a)}{\leq} \sum_{j=1}^h\bk{\norm*{\mathbf{w}_j}_q\norm*{\mathbf{x}_i}_p}^{kp}\\
    &\stackrel{(b)}{\leq} B^{kp}\sum_{j=1}^h\norm*{\vect{w}_j}_q^{kp}\\
    &\stackrel{(c)}{\leq} B^{kp}\norm*{W_1}_{q,p}^{kp}
    \leq \bk{\frac{1}{2}}^{kp} = B'^p,
\end{align*}
where $B' = \bk{\frac{1}{2}}^k$, which is independent of $B$. Here $(a)$ follows from Hölder's inequality, $(b)$ follows from $\norm*{W_1}_{q,kp}^{kp}=\sum_{j=1}^h\norm*{\vect{w}_j}_q^{kp}$, and $(c)$ follows from $\norm*{W_1}_{q,p}^{kp}\geq \norm*{W_1}_{q,kp}^{kp}$. This requires $\norm*{W_2}_{q,p}\leq \frac{1}{2B'}=2^{k-1}$. The conclusion holds for all $W_i$ when $i\geq 2$. It shows that as long as the weight matrix of the first layer satisfies certain conditions related with original input, the constraints on other layers only depends on the activation of each layer. So we can get the overall condition: $\norm*{W_j}_{q,p} \leq 2^{k-1}$ for $j\geq 2$; $\norm{W_1}_{q,p} \leq \frac{1}{2B}$, $\norm{\vect{x}_i}_{p} \leq B$. This allows us to use Lemma~\ref{modified:lemma1} and Lemma~\ref{modified:lemma2} and the proof is completed.
\end{proof}
\begin{remark}
For simplicity, we consider the activations of each layer to be the same in Theorem~\ref{thm:dual_dep} as $\sigma(x)=x^k$. We can easily extend it to the case where activations are different in each layer. It will change nothing but the norm restriction of each weight matrix with respect to the value of $k$. And a similar conclusion would also hold for all activations satisfying property $\sigma(x)\leq x^k$.
\end{remark}

\subsection{Depth-Independent Sample Complexity for Polynomial Neural Networks}
\label{sec:independent}
We next show how to get a depth-independent bound for polynomial networks following a similar argument in~\cite{golowich2018size}.
\begin{lemma}\label{lem:alt-net-lem-2}
  For any matrix $W\in\R^{h\times n}$, $\forall p,q\geq 1$ and $\frac{1}{p}+\frac{1}{q}=1$, there exists a rank-1 matrix $\widetilde{W}$ of the same size as $W$ such that 
\begin{align}
	\norm{\widetilde{W}}_{q,p} \leq \norm{W}_{q,p} ,\; \norm{W-\widetilde{W}}_{q,p} = \bk{\norm{W}_{q,p}^p - \norm{\widetilde{W}}_{q,\infty}^p }^{\frac{1}{p}}. 
\end{align} 
\end{lemma}
\begin{proof}
Let $\widetilde{W}$ be a rank-1 matrix of the same size as $W$ containing only one non-zero row $\mathbf{w}_i$, which is the row with largest $L_q$ norm in $W$, meaning that $i=\argmax_{j\in[h]}\norm*{\textbf{w}_j}_q$. Then by definition, we know that 
\begin{align*}
\norm{W-\widetilde{W}}_{q,p}^p &= \sum_{j\neq i}\norm*{\mathbf{w}_j}_q^p \\
&= \sum_{j=1}^h\norm*{\mathbf{w}_j}_q^p - \norm*{\mathbf{w}_i}_q^p \\
&= \norm*{W}_{q,p}^p - \norm*{W}_{q,\infty}^p.
\end{align*}
So we can have $\norm{W-\widetilde{W}}_{q,p} = \bk{\norm{W}_{q,p}^p - \norm{\widetilde{W}}_{q,\infty}^p }^{\frac{1}{p}}$. And obviously, we also have $\norm{\widetilde{W}}_{q,p} \leq \norm{W}_{q,p}$.
\end{proof}
\begin{remark}
The original proof in~\cite{golowich2018size} uses SVD decomposition to construct the rank-1 matrix. We change the analysis to the row with largest $L_q$ norm in our case.
\end{remark}

\begin{lemma}\label{lem:alt-net-lem-1}
For two polynomial networks $\sfN^{(d)}: \hua*{W_1, \ldots, W_r, \ldots, W_d}$ and $\widetilde{\sfN}^{(d)}: \hua*{W_1, \ldots, \widetilde{W_r}, \ldots, W_d}$, they only differ in the $r$-th layer. If all conditions in Theorem~\ref{thm:dual_dep} hold for both $\sfN^{(d)}$ and $\widetilde{\sfN}^{(d)}$, we have
  \begin{align}
   \sup_{\vect{x} \in \mathcal{X}} \norm*{\sfN^{(d)}(\vect{x}) - \widetilde{\sfN}^{(d)}(\vect{x}) }_p \leq B \prod_{j=1}^d \norm{W_j}_{q,p} \cdot \frac{\norm{W_r - \widetilde{W}_r}_{q,p}}{\norm{W_r}_{q,\infty}}
  \end{align}
\end{lemma}%
\begin{proof}
  Since input $\vect{x}$ and weight matrices satisfy the constraints, every activation function $\sigma_i$ is 1-Lipschitz in $L_p$ norm sense ($p \geq 1$), thus the Lipschitz constant of function $\sfN_{W_{r+1}^d}$ is at most $\prod_{i=r+1}^d \norm*{W_i}_{q,p}$.
  \begin{align*}
    \norm*{\sfN^{(d)}(\vect{x}) - \widetilde{\sfN}^{(d)}(\vect{x}) }_p & = \norm*{ \sfN_{W_{r+1}^d}\bk{\sigma_r\bk{W_r\sigma_{r-1}\bk{ \sfN_{W_1^{r-1}} (\vect{x}) }}}  -  \sfN_{W_{r+1}^d}\bk{\sigma_r\bk{\widetilde{W}_r\sigma_{r-1}\bk{ \sfN_{W_1^{r-1}}(\vect{x}) }} } }_p \\
      & \leq \prod_{i=r+1}^d \norm*{W_i}_{q,p} \norm*{ \sigma_r\bk{W_r\sigma_{r-1}\bk{ \sfN_{W_1^{r-1}}(\vect{x}) }}  -  \sigma_r\bk{\widetilde{W}_r\sigma_{r-1}\bk{ \sfN_{W_1^{r-1}}(\vect{x}) }} }_p \\
      & \leq \prod_{i=r+1}^d \norm*{W_i}_{q,p} \norm*{ W_r\sigma_{r-1}\bk{ \sfN_{W_1^{r-1}}(\vect{x}) }  -  \widetilde{W}_r\sigma_{r-1}\bk{ \sfN_{W_1^{r-1}}(\vect{x}) } }_p \\
      & \leq \prod_{i=r+1}^d \norm*{W_i}_{q,p} \cdot \norm*{W_r - \widetilde{W}_r}_{q,p}  \cdot \prod_{j=1}^{r-1} \norm*{W_j}_{q,p} \cdot \norm*{\vect{x}}_p \\
      & \leq B \prod_{j=1}^d \norm{W_j}_{q,p} \cdot \frac{\norm{W_r - \widetilde{W}_r}_{q,p}}{\norm{W_r}_{q,p}} \\
      & \leq B \prod_{j=1}^d \norm{W_j}_{q,p} \cdot \frac{\norm{W_r - \widetilde{W}_r}_{q,p}}{\norm{W_r}_{q,\infty}} \qedhere
  \end{align*}
\end{proof}

\begin{lemma}\label{lem:alt-net-lem-3}
  If $\sfN^{(d)}$ satisfy $\prod_{j=1}^d \norm{W_j}_{q,\infty} \geq \Gamma $, $\prod_{j=1}^d \norm{W_j}_{q,p} \leq M $. Then for any $r \in [d]$, we have
  \begin{align}
   \min_{j \in \hua{1,\cdots,r} } \frac{\norm{W_j}_{q,p} }{ \norm{W_j}_{q,\infty} } \leq \bk{\frac{M}{\Gamma}}^{\frac{1}{r}}.
  \end{align}
\end{lemma}%
\begin{proof}
  The proof is similar to the proof of Lemma 6 in~\cite{golowich2018size}. Note that $\norm*{W_j}_{q,p} \geq \norm*{W_j}_{q,\infty} \geq 0$ for $p\geq 1$. So we have
  \begin{align*}
      \frac{M}{\Gamma} \geq \frac{\prod_{j=1}^{d}\left\|W_{j}\right\|_{q,p}}{\prod_{j=1}^{d}\left\|W_{j}\right\|_{q,\infty}}=\prod_{j=1}^{d} \frac{\left\|W_{j}\right\|_{q,p}}{\left\|W_{j}\right\|_{q,\infty}} \geq \prod_{j=1}^{r} \frac{\left\|W_{j}\right\|_{q,p}}{\left\|W_{j}\right\|_{q,\infty}} \geq\left(\min _{j \in\{1, \ldots, r\}} \frac{\left\|W_{j}\right\|_{q,p}}{\left\|W_{j}\right\|_{q,\infty}}\right)^{r},
  \end{align*}
  which completes the proof.
\end{proof}

\begin{theorem}\label{thm:alternative-net}
  For any polynomial network $\sfN^{(d)}$ such that $\prod_{j=1}^d \norm{W_j}_{q,\infty} \geq \Gamma$ and $\prod_{j=1}^d\norm{W_j}_{q,p} \leq \prod_{j=1}^d M(j)= M $, and for $j \in \hua{2,\cdots,d}$, $M(j) \leq 2^{k-1}$; $\norm{W_1}_{q,p} \leq \frac{1}{2B}$, $\norm{\vect{x}_i}_{p} \leq B$, there exists an alternative net $\widetilde{\sfN}^{(d)} $ such that they are identical except for weights of the $r$-th layer for $\forall r \in [d]$: $\widetilde{W}_r$, where $\widetilde{W}_r = \vect{a} \vect{b}^\transpose$, where $\vect{b}$ is the $m$-th row of $W_r$ with the largest $L_q$ norm,  and $\vect{a}$ is a one-hot vector $\vect{a} = \delta_{m}$. And we have
  \begin{align}
   \sup_{\vect{x}} \norm*{\sfN^{(d)}(\vect{x}) - \widetilde{\sfN}^{(d)}(\vect{x}) }_p \leq B \prod_{j=1}^d \norm{W_j}_{q,p} \bk{\frac{2p\log(M/\Gamma)}{r}}^{\frac{1}{p}}. 
  \end{align}
\end{theorem}

\begin{proof}
  Combine Lemma \ref{lem:alt-net-lem-1}, \ref{lem:alt-net-lem-2} and \ref{lem:alt-net-lem-3}, we have
  \begin{align*} 
  	\sup_{\vect{x} \in \mathcal{X}} \norm*{\sfN^{(d)}(\vect{x}) - \widetilde{\sfN}^{(d)}(\vect{x}) }_p & \leq B \prod_{j=1}^d \norm{W_j}_{q,p} \cdot \frac{\norm{W_r - \widetilde{W}_r}_{q,p} }{\norm{W_r}_{q,\infty}}  \tag*{(Lemma \ref{lem:alt-net-lem-1})}\\
      & \leq B \prod_{j=1}^d \norm{W_j}_{q,p} \cdot \bk{\frac{\norm{W_r}_{q,p}^p - \norm{\widetilde{W}_r}_{q,\infty}^p }{\norm{W_r}_{q,\infty}^p}}^{\frac{1}{p}} \tag*{(Lemma \ref{lem:alt-net-lem-2})} \\
      & \leq B \prod_{j=1}^d \norm{W_j}_{q,p} \cdot \bk{\frac{\norm{W_r}_{q,p}^p}{\norm{W_r}_{q,\infty}^p} - 1}^{\frac{1}{p}} \\
      & \leq B \prod_{j=1}^d \norm{W_j}_{q,p} \cdot \bk{ \bk{\frac{M}{\Gamma}}^{p/r} - 1 }^{\frac{1}{p}} \tag*{(Lemma \ref{lem:alt-net-lem-3})}\\
      & \leq B \prod_{j=1}^d \norm{W_j}_{q,p} \cdot \bk{\exp\bk{\frac{p}{r}\log(M/\Gamma)} - 1 }^{\frac{1}{p}} \\
      & \stackrel{(a)}{\leq} B \prod_{j=1}^d \norm{W_j}_{q,p} \cdot \bk{\frac{ 2p\log(M/\Gamma) }{ r }}^{\frac{1}{p}}.  \qedhere  \\
  \end{align*}
  Here $(a)$ follows from the conclusion shown in~\cite{golowich2018size}.
\end{proof}

\begin{remark}
  Theorem \ref{thm:alternative-net} shows that the original neural network can be approximated with composition of two networks, 
  \begin{align*}
    \vect{x} & \mapsto W_d\sigma_{d-1}(W_{d-1}\cdots \sigma_r(\norm{\vect{b}}_q \vect{a} \bk{\frac{\vect{b}}{\norm{\vect{b}}_q}}^\transpose \cdots \sigma_1(W_1\vect{x}) \cdots ) \cdots) \\
    & \Downarrow \\
    \vect{x} & \mapsto \bk{\frac{\vect{b}}{\norm{\vect{b}}_q}}^\transpose \sigma_{r-1}(W_{r-1}(\cdots \sigma_1(W_1\vect{x}) \cdots )) \\
    x & \mapsto W_d\sigma_{d-1}(W_{d-1}(\cdots \sigma_r\bk{\norm{\vect{b}}_q \vect{a} x} \cdots )),
  \end{align*}
  where $\vect{b}$ is the $m$-th row of $W_r$ with the largest $L_q$ norm, and $\vect{a}$ is a one-hot vector $\vect{a} = \delta_{m}$.
\end{remark}

\begin{theorem}\label{thm:composition-net}
Let $\Hil: \R^n \to [-R, R] $, $\huaF_{L,a}: [-R, R] \to \R $ which is $L$-Lipschitz and $f(0) = a $ for some fixed $a$, then we have
  \begin{align}
   \hat{\huaR}_m\bk{ \huaF_{L,a} \circ \Hil } \leq cL\bk{ \frac{R}{\sqrt{m}} + \log^{3/2}(m)\hat{\huaR}_m(\Hil) },
  \end{align}
  where $c>0$ is a universal constant.
\end{theorem}
\begin{proof}
  This is a general result which is independent of any specify activation functions. So the proof is the same as Theorem 4 in~\cite{golowich2018size} and is omitted here for brevity.
\end{proof}


\begin{theorem}\label{thm:r-dep-bound}
Considering the following hypothesis class of networks on $\mathcal{X}=\{\vect{x}:\norm{\vect{x}}\leq B\}$
  \begin{equation*}
    \Hil = \left\{ \sfN^{(d)} : 
      \begin{aligned}
         & \prod_{j=1}^d \norm{W_j}_{q,\infty} \geq \Gamma \\
         & \forall j \in \hua{1, \ldots, d}, W_j \in \mathcal{W}_j, \norm{W_j}_{q,p} \leq M(j)
      \end{aligned}
    \right\},
  \end{equation*}
  for some $\Gamma \geq 1$, $\norm*{W_1}_{q,p} \leq \frac{1}{2B}$ and $\norm*{W_i}_{q,p} \leq 2^{k-1}, i \in \hua{2,\cdots, d}$. For any $r \in \hua{1, \cdots, d} $, define
  \begin{equation*}
    \Hil_r = \left\{ \sfN^{(r)} : 
      \begin{aligned}
         & \sfN^{(r)} : \R^d \to \R \\
         & \forall j \in \hua{1, \cdots, r-1}, W_j \in \mathcal{W}_j; \\
         & \forall j \in \hua{1, \cdots, r}, \norm{W_j}_{q,p} \leq M(j)
      \end{aligned}
    \right\}.
  \end{equation*}
  Define $\ell\circ \Hil = \hua{(\ell_1(h(\vect{x}_1)),\cdots,\ell_m(h(\vect{x}_m))) : h \in \Hil } $, where $\ell_1,\cdots,\ell_m$ are real-valued loss functions which are $\frac{1}{\gamma}$-Lipschitz and satisfy $\ell_1(0) = \ell_2(0) = \ldots = \ell_m(0) =a$ for some $a$ such that $\abs{a}\leq \frac{B\prod_{j=1}^d M(j)}{\gamma}$. Then the sample complexity $\hat{\huaR}_m(\ell\circ\Hil)$ is upper bounded by
  \begin{align*}
   \frac{cB\prod_{j=1}^d M(j)}{\gamma} \bk{ \min_{r\in[d]} \hua*{\frac{\log^{\frac{3}{2}}(m)\cdot\hat{\huaR}_m(\Hil_r)}{B\prod_{j=1}^r M(j)} + \bk{\frac{\log\bk{\prod_{j=1}^d M(j)/\Gamma}}{r}}^{\frac{1}{p}} + \frac{1+\sqrt{\log r}}{\sqrt{m}} }}.
   \end{align*}
\end{theorem}
\begin{proof}
  Combining Theorem~\ref{thm:alternative-net} and~\ref{thm:composition-net} and following the same line of argument in the proof of Theorem 5 in~\cite{golowich2018size}, we can arrive the similar results shown above. Note that the difference here is that we are considering a more general form of norm constraints for network parameters $\norm*{W_j}_{q,p}$ which depends on the choice of activation function $\sigma(x)=x^k \;(k\geq 2)$.
\end{proof}

Using the above result and some other tricks, it's easy to show that polynomial networks can also have depth-independent bound, which holds under such constraints of arbitrary $q,p$-norm of the weight matrices.
\begin{corollary}\label{col:main}
  Let $\Hil$ be a class of depth-$d$ polynomial network, with weight matrices satisfying $\norm{W_j}_{q,p} \leq M(j), j\in[d]$, and loss function and $\Hil$ satisfying the conditions of \ref{thm:r-dep-bound}, and $\norm*{\vect{x}}_p\leq B$, $\norm*{W_1}_{q,p} \leq \frac{1}{2B}$ and $\norm*{W_i}_{q,p} \leq 2^{k-1}, i \in \hua{2,\cdots, d}$. It holds that
\begin{align}
    \hat{\mathcal{R}}_{m}(\ell \circ \mathcal{H}) \leq \mathcal{O}\left(\frac{B \prod_{j=1}^{d} M(j)}{\gamma} \cdot \min \left\{\frac{\bar{\log} ^{\frac{3}{4}}(m) \sqrt{\bar{\log} \left(\prod_{j=1}^{d} M(j)/\Gamma\right)}}{m^{\frac{1}{4}}}, \sqrt{\frac{d}{m}}\right\}\right),
\end{align}
where $\bar{\log}(z):=\max\{1,\log(z)\}$.
\end{corollary}

\begin{proof}
  From Theorem~\ref{thm:dual_dep} we know that $\hat{\mathcal{R}}_{m}(\ell \circ \mathcal{H}) \leq \mathcal{O}\left(\frac{B \prod_{j=1}^{d} M(j)}{\gamma} \sqrt{\frac{d}{m}}\right)$. Combine this with Lemma 3 in~\cite{golowich2018size} we can have the above result.
\end{proof}

\begin{remark}
  If we assume that $\prod_{j=1}^{d} M(j)$ is upper bounded by a constant, the first term in the $\min$ argument is independent of the depth $d$. Although this may be a very strong assumption in practice~\cite{golowich2018size}, it leads to a depth-independent upper bound for sample complexity of polynomial neural networks.
\end{remark}

\section{Conclusion and Future Work}
In this paper, we study the sample complexity of one special type of neural networks: polynomial neural networks (PNNs), which uses polynomial functions as activations. We discuss both depth-dependent and depth-independent upper bound for sample complexity of PNNs, following a similar line of analysis in~\cite{golowich2018size}. There are many other interesting applications of this depth-independent bound shown in~\cite{golowich2018size}, which can be possible future directions for this paper.



\bibliographystyle{apalike}
\bibliography{document}
\newpage

\section{Supplement}
\label{sec:supp}
Here is a special case where we consider $q=1$ and $p=\infty$.
\begin{lemma} \label{sup:lemma1}
  Let $\sigma$ be a 1-Lipschitz, positive-homogeneous activation function which is applied element-wise. Then for any vector-valued class $\mathcal{F}$, and any convex and monotonically increasing function $g:\mathbb{R}\rightarrow[0,\infty)$, we have
  \begin{align}
  \E \sup_{f\in\huaF, \norm{W}_{1,\infty}\leq R} g\bk{\norm*{\sum_{i=1}^m \eps_i \sigma\bk{W f(\vect{x}_i}) }_\infty } \leq  
      2\E \sup_{f\in\huaF} g\bk{R \norm*{\sum_{i=1}^m \eps_i f(\vect{x}_i) }_\infty }
  \end{align}
\end{lemma}
\begin{proof}
  \begin{align*}
    \E \sup_{f\in\huaF, \norm{W}_{1,\infty}\leq R } g\bk{\norm*{\sum_{i=1}^m \eps_i \sigma\bk{W f(\vect{x}_i)} }_\infty } & = 
        \E \sup_{f\in\huaF,\, \max_{j}\norm{\vect{w}_j}_{1} \leq R } g\bk{ \abs*{\sum_{i=1}^m \eps_i \sigma\bk{\vect{w}_j^T f(\vect{x}_i)}} } \\
      & = \E \sup_{f\in\huaF,\, \norm{\vect{w}}_{1} = R } g\bk{ \abs*{ \sum_{i=1}^m \eps_i \sigma\bk{\vect{w}^T f(\vect{x}_i)}} } \\
      & \leq \E \sup_{f\in\huaF,\, \norm{\vect{w}}_{1} = R } g\bk{ \sum_{i=1}^m \eps_i \sigma\bk{\vect{w}^T f(\vect{x}_i)}} + \\
      & \phantom{\leq}\quad \E \sup_{f\in\huaF,\, \norm{\vect{w}}_{1} = R } g\bk{ -\sum_{i=1}^m \eps_i \sigma\bk{\vect{w}^T f(\vect{x}_i)}}  \\
      & = 2 \E \sup_{f\in\huaF,\, \norm{\vect{w}}_{1} = R } g\bk{ \sum_{i=1}^m \eps_i \sigma\bk{\vect{w}^T f(\vect{x}_i)}}  \\
      & \leq 2 \E \sup_{f\in\huaF,\, \norm{\vect{w}}_{1} = R } g\bk{ \sum_{i=1}^m \eps_i \vect{w}^T f(\vect{x}_i)}  \\
      & \leq 2 \E \sup_{f\in\huaF,\, \norm{\vect{w}}_{1} = R } g\bk{ \norm{\vect{w}}_1 \norm*{\sum_{i=1}^m \eps_i f(\vect{x}_i)}_\infty }  \\
      & \leq 2 \E \sup_{f\in\huaF} g\bk{ R \norm*{\sum_{i=1}^m \eps_i f(\vect{x}_i)}_\infty } \qedhere \\ 
  \end{align*}
\end{proof}
\begin{remark}
Note that Lemma~\ref{sup:lemma1} has the exact same form as Lemma~\ref{modified:lemma1}. From that we can continue the analysis as shown in Lemma~\ref{modified:lemma2} and Theorem~\ref{thm:dual_dep}.
\end{remark}

\end{document}